\documentclass[nojss, nofooter]{jss}

\usepackage{orcidlink,lmodern}
\usepackage{amsthm}
\usepackage{amsmath}
\usepackage{algorithm}
\usepackage{algpseudocode}
\usepackage{placeins}
\usepackage{subcaption}
\usepackage{amssymb}
\usepackage{booktabs}

\newtheorem{theorem}{Theorem}

\author{
  Zhongli Jiang~\orcidlink{0000-0002-3396-0345}\\
  Wenzhou University
  \And
  Min Zhang\\
  UC Irvine
  \And
  Dabao Zhang~\orcidlink{0000-0003-0629-8672}\\
  UC Irvine
}

\Plainauthor{Zhongli Jiang, Min Zhang, Dabao Zhang}

\title{qshap: Fast Shapley Decomposition of \texorpdfstring{$R^2$}{R-squared} for Gradient-Boosted Trees}
\Plaintitle{qshap: Fast Shapley Decomposition of R-squared for Gradient-Boosted Trees}

\Abstract{
Numerous methods have been developed to quantify feature attributions in individual predictions for tree ensembles. However, many applications require global measures of feature contributions to overall model performance. Although local attribution scores can be aggregated to characterize feature importance, such summaries do not directly decompose measures of predictive performance, such as $R^2$. This article introduces \pkg{qshap}, available in both \proglang{R} and \proglang{Python}, which provides Shapley decomposition of $R^2$ values for gradient-boosted decision trees (GBDTs) to quantify feature-specific contributions to model performance. By decomposing the quadratic loss of individual observations, \pkg{qshap} provides flexible tools to explore the importance of individual features and observations. \pkg{qshap} currently supports widely used GBDT implementations, including \pkg{xgboost}, \pkg{lightgbm}, and \pkg{catboost}, through a unified tree representation and efficient \proglang{C++} backends. Its modular design can accommodate other GBDT implementations built from binary decision trees. In addition, we introduce a specialized backend for oblivious trees that exploits their symmetric structure to substantially accelerate computation.
}

\Keywords{Shapley values, feature-specific $R^2$, gradient-boosted decision trees, explainable artificial intelligence, interpretable machine learning}
\Plainkeywords{Shapley values, feature-specific R-squared,
gradient-boosted decision trees, explainable artificial intelligence,
interpretable machine learning}

\Address{
  Zhongli Jiang\\
  College of Computer Science and Artificial Intelligence\\
  Wenzhou University\\
  Wenzhou, Zhejiang 325035, China\\
  E-mail: \email{zhongli.jiang.stats@gmail.com}\\
  URL: \url{https://github.com/catstats}\\[1em]

  Min Zhang, Dabao Zhang\\
  Department of Epidemiology \& Biostatistics\\
  Joe C. Wen School of Population \& Public Health\\
  University of California, Irvine\\
  Irvine, California 92617, United States of America\\
  E-mail: \email{minzhang@uci.edu} (Min Zhang)\\
  E-mail: \email{zdb969@hs.uci.edu} (Dabao Zhang)
}

\begin{document}



\section[Introduction]{Introduction} \label{sec:intro}

Gradient-boosted decision trees (GBDTs) \citep{friedman2001greedy}, including XGBoost \citep{chen2016xgboost}, LightGBM
\citep{ke2017lightgbm}, and CatBoost \citep{prokhorenkova2018catboost},
are widely used for tabular data because of their predictive
performance and scalability.
However, interpreting the contributions of individual features remains challenging, particularly as tree depth and ensemble size increase.

Several notions of feature importance are available for fitted tree ensembles, but they answer different questions. Built-in summaries such as split counts and gain importance describe how variables are used by the fitted trees, whereas permutation importance measures the deterioration of a chosen performance criterion after perturbing a variable \citep{breiman2001random, fisher2019all}. To address these limitations, Shapley additive explanations (SHAP; \citealp{lundberg2017unified}) leverage Shapley values \citep{shapley1953value} to provide robust feature attributions. Its specialized implementation TreeSHAP \citep{lundberg2020local} further introduces a fast and exact polynomial-time algorithm for tree ensembles. The \proglang{Python} package \pkg{shap} provides a widely used implementation of SHAP and TreeSHAP. Related software includes the \proglang{R} package \pkg{treeshap}, which provides a unified TreeSHAP interface for several tree ensemble implementations \citep{mayer2026treeshap}, and the packages \pkg{shapr} and \pkg{shaprpy}, which emphasize conditional coalition values that account for dependency among the input variables \citep{jullum2026shapr}. Nonetheless, these methods are designed to explain individual predictions, whereas many applications require a global measure of feature contribution, especially in the fields of biomedical science and finance. Although local SHAP values are often aggregated over observations, for example, by aggregating absolute values \citep{lundberg2020local}, they typically measure average attribution across local quantities, rather than an exact decomposition of model-level goodness of fit.

To characterize global model behavior, several methods apply the Shapley value allocation to measures of predictive performance. In linear regression, \citet{lipovetsky2001analysis} decompose the model $R^2$ and quantify the relative importance of predictors in the presence of multicollinearity. A closely related formulation is the LMG measure implemented in the \proglang{R} package \pkg{relaimpo}, which averages the incremental $R^2$ attributed to each predictor over all possible predictor orderings \citep{groemping2006relative}. Although these methods provide fair allocations of explained variance, the evaluation generally requires considering permutations that grow exponentially with the feature dimensions, and \pkg{relaimpo} is only available for linear models. 

Beyond linear regression, \pkg{sage-importance} implements the SAGE algorithm, which utilizes Shapley allocation based on expected loss reduction and estimates Shapley value feature importance by Monte Carlo sampling over feature subset permutations \citep{covert2020understanding}. Similarly, the \proglang{Python} package \pkg{vimpy} and the \proglang{R} package \pkg{vimp} implement SPVIM, a resampling-based framework for estimating Shapley global variable importance that also supports statistical inference \citep{williamson2020efficient}. Shapley effects provide a closely related population-level, variance-based formulation \citep{song2016shapley, owen2017shapley}. While these methods provide general, model-agnostic measures of global feature importance, their computation may require feature-subset sampling, nuisance function estimation, cross-fitting, or repeated model fitting. Their computational cost can therefore become substantial, particularly in high-dimensional settings. The SHAFF method proposed by \citet{benard2022shaff} uses random forests to guide feature-subset sampling and estimate the explained-variance quantities required for Shapley effects. However, SHAFF relies on sampling and is specifically designed for random forests, and thus does not directly provide an $R^2$ decomposition for boosted tree ensembles.

In this article, we present a deterministic framework, Q-SHAP \citep{pmlr-v286-jiang25a}, for computing feature-specific Shapley values of explained variance (i.e., $R^2$) for fitted GBDT models. By exploiting the internal tree structure, Q-SHAP enables fast and exact decomposition of $R^2$ in polynomial time without any sampling or model refitting. In addition, we introduce a specialized algorithm for oblivious trees, in complexity $O(nD + \min(n, L)LD)$ for $n$ samples, where $L$ is the maximum number of leaves and $D$ denotes the maximum depth, matching the asymptotic complexity of Linear TreeSHAP \citep{bifet2022linear}. We further introduce the \pkg{qshap} software, implemented in both \proglang{R} and \proglang{Python}, with an efficient backend in \proglang{C++}, supporting large-scale feature attribution with parallel computation, and providing rich visualization tools through a simple, user-friendly interface.

The rest of the paper is organized as follows. Section~\ref{sec:methods} introduces the theoretical background of Shapley value and details of the Q-SHAP methods. Section~\ref{sec:r_package} describes the design and usage of the \proglang{R} package. Section~\ref{sec:python_package} discusses the \proglang{Python} version of our package. We further discuss the internals and extensibility in Section~\ref{sec:internal}, and we conclude the paper in Section~\ref{sec:conclusion}.


\section{Methodological foundations} \label{sec:methods}

The package \pkg{qshap} implements and extends the Q-SHAP methodology introduced by \citet{pmlr-v286-jiang25a}. The original method establishes three main results. First, feature-specific \(R^2\) can be formulated as a Shapley allocation of explained variation, or equivalently of reductions in squared error. Second, after expanding the squared-error loss, each feature contribution separates into a linear combination of a conventional SHAP term for the prediction and a quadratic SHAP term for the squared prediction. Third, for decision trees, the quadratic term can be evaluated exactly in polynomial time by aggregating over pairs of leaves rather than enumerating all feature coalitions.

In this article, we summarize the identities and algorithms required to understand the software and focus on their computational realization. In addition, \pkg{qshap} introduces a specialized efficient algorithm for oblivious trees. Complete theoretical derivations and proofs of the original Q-SHAP representation are given in \citet{pmlr-v286-jiang25a}.

\subsection{Feature-specific decomposition of model fits} \label{sec:rsq_definition}

At the population level, let
\(X=(X_1,\ldots,X_p)\) and
\(\mathcal P=\{1,\ldots,p\}\) be the full set of features.
Following \citet{pmlr-v286-jiang25a}, for any coalition
\(S\subseteq\mathcal P\), define the oracle predictor
\[
m_S(x)
=
\mathbb{E}\!\left(Y\mid X_S=x_S\right).
\]
The corresponding population coefficient of determination is
\[
R_S^2
=
\frac{\operatorname{Var}\{m_S(X_S)\}}
     {\operatorname{Var}(Y)}.
\]

For feature $j$ and coalition
$S\subseteq\mathcal P\setminus\{j\}$, let
\[
\omega_p(S)
=
\frac{|S|!(p-|S|-1)!}{p!}
=
\frac{1}{p}
\binom{p-1}{|S|}^{-1}
\]
be the Shapley weight. The oracle feature-specific contribution to
$R^2$ is then
\[
\phi_j^{R^2}
=
\sum_{S\subseteq\mathcal P\setminus\{j\}}
\omega_p(S)
\left(
R_{S\cup\{j\}}^2-R_S^2
\right).
\]
Let $\mathcal{D}= \{(y_i,x_{i\cdot})\}_{i=1}^n$
denote the data to be explained.  Under the squared-error loss, define
\[
Q_S
=
\sum_{i=1}^n
\left\{
y_i-\widehat m_S(x_{i\cdot})
\right\}^2,
\]
where $\widehat m_S(x_{i\cdot})$ is the prediction of $m_S(x_{i\cdot})$ with features in $S$, and the no-feature prediction is the sample mean,
\(\widehat m_{\emptyset}(x_{i\cdot})=\bar y\), where
\(\bar y=n^{-1}\sum_i y_i\). Therefore,
\[
Q_{\emptyset}
=
\sum_{i=1}^n(y_i-\bar y)^2
\]
is the total sum of squares, and the fitted model has the coefficient of determination
\[
\widehat R^2
=
1-\frac{Q_{\mathcal P}}{Q_{\emptyset}}.
\]

The feature-specific contribution to model \(R^2\) can be estimated by
\begin{align}
\widehat 
\phi^{R^2}_j
&=
-\frac{1}{Q_{\emptyset}}
\sum_{S\subseteq\mathcal{P}\setminus\{j\}}
\omega_p(S)
\left(
Q_{S\cup\{j\}}-Q_S
\right).
\label{eq:feature_rsq}
\end{align}
By the Shapley efficiency axiom,
\[
\sum_{j=1}^p \widehat \phi^{R^2}_j = 1-\frac{Q_{\mathcal P}}{Q_{\emptyset}}
=
\widehat R^2.
\]
Thus, the model \(R^2\) is decomposed additively across the input features.

\subsection{Linear and quadratic Shapley values for trees} \label{sec:tree_subset_prediction}

Consider a regression tree with \(L\) leaves and maximum depth
\(D\). Let \(\widehat m^\ell\) denote the prediction associated
with leaf \(\ell\), and let \(\mathcal F^\ell\) be the set of
features appearing along its root-to-leaf path. The full tree
prediction can be written as
\[
\widehat m_{\mathcal P}(x_{i\cdot})
=
\sum_{\ell=1}^L
\widehat m^\ell\times
\mathbf{1}\{x_{i\cdot}\text{ reaches leaf }\ell\}.
\]

For a subset \(S\), the path-dependent prediction follows the
observed split decisions for features in \(S\) and averages over
the branches associated with features outside \(S\), using the
training-sample proportions stored in the tree. Let
\(n_\ell\) be the number of training observations reaching leaf
\(\ell\), and define
\[
\widehat m_{\emptyset}^{\ell} = \widehat m^\ell\times \frac{n_\ell}{n}.
\]
For feature \(j\in\mathcal F^\ell\), let
\(w_j^\ell(x_{i\cdot})\) denote the multiplicative change in the
weight of leaf \(\ell\) when feature \(j\) is added to the
coalition. If a feature appears multiple times along the path,
\(w_j^\ell(x_{i\cdot})\) combines the corresponding factors from
all of its appearances. If an observation disagrees with any split on feature $j$ along the path, then $w_j^\ell(x_{i\cdot})=0$. The subset prediction can then be
expressed as
\begin{equation}
\widehat m_S(x_{i\cdot})
=
\sum_{\ell=1}^L
\widehat m_{\emptyset}^{\ell}\times
\prod_{j\in S\cap\mathcal F_\ell}
w_j^\ell(x_{i\cdot}).
\label{eq:tree_subset_prediction}
\end{equation}

Equation~\eqref{eq:feature_rsq} can be evaluated by expanding its
observation-level loss differences. Define
\begin{align}
T_{1,ij}
&=
\sum_{S\subseteq\mathcal{P}\setminus\{j\}}
\omega_p(S) \times
\left\{
\widehat m_{S\cup\{j\}}(x_{i\cdot})
-
\widehat m_S(x_{i\cdot})
\right\},
\label{eq:t1_definition}
\\
T_{2,ij}
&=
\sum_{S\subseteq\mathcal{P}\setminus\{j\}}
\omega_p(S)\times
\left\{
\left[\widehat m_{S\cup\{j\}}(x_{i\cdot})\right]^2
-
\left[\widehat m_S(x_{i\cdot})\right]^2
\right\}.
\label{eq:t2_definition}
\end{align}
Here, \(T_{1,ij}\) is the ordinary path-dependent TreeSHAP
contribution, whereas \(T_{2,ij}\) is the Shapley contribution of
feature \(j\) to the squared subset prediction. Expanding the
squared loss gives
\begin{equation}
\hat{\phi}^{R^2}_j
=
\frac{1}{Q_{\emptyset}}
\sum_{i=1}^n
\left(
2y_iT_{1,ij}-T_{2,ij} \right).
\label{eq:rsq_t0_t2}
\end{equation}
Consequently, once \(T_{1,ij}\) and \(T_{2,ij}\) are available,
the feature-specific \(R^2\) values follow by simple aggregation
over observations.

The linear contribution \(T_{1,ij}\) can be calculated using TreeSHAP. Several faster algorithms have subsequently been proposed \citep{yang2021fast, karczmarz2022improved, bifet2022linear, mohammadi2026quadrashap, wettenstein2026quadrature}. The main computational problem is therefore the quadratic contribution \(T_{2,ij}\). Squaring
Equation~\eqref{eq:tree_subset_prediction} yields
\begin{align}
\left[\widehat m_S(x_{i\cdot})\right]^2
=
\sum_{\ell_1=1}^L\sum_{\ell_2=1}^L
\widehat m_{\emptyset}^{\ell_1}
\widehat m_{\emptyset}^{\ell_2}
\prod_{k\in S}
w_k^{\ell_1}(x_{i\cdot})w_k^{\ell_2}(x_{i\cdot}),
\label{eq:leaf_pair_expansion}
\end{align}
where weights associated with features absent from a path are
defined as one. This expansion shows that the quadratic game can
be represented through interactions between pairs of leaves.
For arbitrary binary trees, the general Q-SHAP algorithm
aggregates these leaf-pair terms in
\(O(L^2D^2)\) time per tree and observation. Before stating the general tree algorithm, we first introduce necessary polynomial annotations.  For a fixed observation \(x_{i\cdot}\), leaf pair
\((\ell_1,\ell_2)\), and feature \(j\), let
\[
n_{12}=\left|F^{\ell_1}\cup F^{\ell_2}\right|
\]
and define
\[
P_{ij}^{\ell_1\ell_2}(z)
=
\prod_{k\in(F^{\ell_1}\cup F^{\ell_2})\setminus\{j\}}
\left\{
1+
w_k^{\ell_1}(x_{i\cdot})
w_k^{\ell_2}(x_{i\cdot})z
\right\}.
\]
The Shapley weights associated with the coalition sizes are encoded by
\[
C_{n_{12}}(z)
=
\sum_{s=0}^{n_{12}-1}
\binom{n_{12}-1}{s}^{-1}z^s.
\]
For two polynomials of the same degree, \(C(z)\cdot P(z)\) denotes
the inner product of their coefficient vectors.

\subsection{The algorithm for general trees}

For general trees, Q-SHAP uses the stable leaf-pair formulation. This is the default implementation for arbitrary binary trees.

Algorithm~\ref{alg:qshap-general} computes the Shapley value of the quadratic game $S\mapsto\{\widehat m_S(x_{i\cdot})\}^2$ by explicitly expanding the square into leaf-pair interactions. The dot product $C_{n_{12}}(z)\cdot P^{\ell_1\ell_2}(z)$ is evaluated by the numerically stable complex-root implementation via inverse fast Fourier transformation described in the general-tree Q-SHAP derivation. 

\begin{algorithm}[!htbp]
\caption{\pkg{Q-SHAP}}
\label{alg:qshap-general}
\begin{algorithmic}
\State \pkg{Q-SHAP}($\mathbf{x}_{i\cdot}$)
\State Initialize $T[j]=0$ for $j=1, \cdots, p$
\For{$\ell_1$  $\in$ index set ${0, \ldots}, L-1$}
\For {$\ell_2$ $\in$ index set $\ell_1, \ldots, L-1$}
    \State Let $n_{12}=|F^{\ell_1} \cup F^{\ell_2}|$
    \For {$j \in F^{\ell_1} \cup F^{\ell_2}$}
    \State Let  $t[j] =\frac{1}{n_{12}} [w_j^{\ell_1}(\mathbf{x}_{i\cdot})w_j^{\ell_2}(\mathbf{x}_{i\cdot})-1] \times \hat{m}_{\emptyset}^{\ell_1}\hat{m}_{\emptyset}^{\ell_2} [C_{n_{12}}(z) \cdot P^{\ell_1\ell_2}(z)]$
    \If{$\ell_1 \neq \ell_2$}
    \State $T[j] = T[j] + 2t[j]$
    \Else
    \State
    $T[j] = T[j] + t[j]$
    \EndIf
    \EndFor
\EndFor
\EndFor
\State return $T=(T[1],T[2],\cdots, T[p])$
\end{algorithmic}
\end{algorithm}

\FloatBarrier
\subsection{Stagewise decomposition of gradient-boosted tree ensembles} \label{sec:boosting_extension}

The preceding sections establish the decomposition for a single tree. At first sight, extending it to a boosted-tree ensemble appears to require considering the interactions among all pairs of trees. However, by exploiting the sequential nature of gradient boosting, we bridge the squared loss for boosted trees with residuals and avoid this exhaustive expansion.

Consider an ensemble containing \(K\) trees. Let
\(\widehat m_{\mathcal P}^{(k)}(x_{i\cdot})\) be the output of
tree \(k\), and let \(\alpha\) denote its learning rate. The fitted values are updated recursively as
\[
\widehat y_i^{(k)}
=
\widehat y_i^{(k-1)}
+
\alpha\widehat m_{\mathcal P}^{(k)}(x_{i\cdot}),
\qquad
\widehat y_i^{(0)}=\bar y.
\]
Writing
\[
r_i^{(k-1)}
=
y_i-\widehat y_i^{(k-1)}
\]
for the residual before the construction of tree \(k\), the change in squared loss for tree \(k\) is
\begin{align}
\left(r_i^{(k)}\right)^2
-
\left(r_i^{(k-1)}\right)^2
=
\alpha^2
\left[
\widehat m_{\mathcal P}^{(k)}(x_{i\cdot})
\right]^2
-
2\alpha r_i^{(k-1)}
\widehat m_{\mathcal P}^{(k)}(x_{i\cdot}).
\label{eq:boosting_loss_identity}
\end{align}

Equation~\eqref{eq:boosting_loss_identity} breaks the ensemble-wide problem down into a sequence of single-tree problems. Applying Shapley additivity, the observation-level
loss contribution of feature \(j\) at boosting stage \(k\) is
\[
\alpha^2T_{2,ij}^{(k)}
-
2\alpha r_i^{(k-1)}T_{1,ij}^{(k)}.
\]
Summing these contributions over observations and boosting
stages gives
\begin{equation}
\hat{\phi}^{R^2}_j
=
-\frac{1}{Q_{\emptyset}}
\sum_{k=1}^K\sum_{i=1}^n
\left\{
\alpha^2T_{2,ij}^{(k)}
-
2\alpha r_i^{(k-1)}T_{1,ij}^{(k)}
\right\}.
\label{eq:boosted_rsq}
\end{equation}
The residual \(r_i^{(k-1)}\) depends on the predictions of all
previous trees. Consequently, the interactions between tree \(k\) and the
preceding ensemble are incorporated implicitly without
explicitly expanding all pairs of trees.  Summing Equation~\eqref{eq:boosting_loss_identity} over the boosting stages yields a telescoping sum where the intermediate losses cancel, leaving only the overall loss difference of the ensembles and the baseline. Algorithm~\ref {alg:rsq-general} summarizes the stagewise procedure:

\begin{algorithm}[H]
\caption{RSQ-SHAP for boosted trees}
\label{alg:rsq-general}
\begin{algorithmic}
    \State \textbf{RSQ-SHAP}
    \State Initialize $\boldsymbol{\Delta}_Q=\mathbf{0}_p$ and
    $\widehat y_i^{(0)}=\bar y$ for $i=1,\ldots,n$
    \For{$k=1,\ldots,K$}
        \For{$i=1,\ldots,n$}
            \State $r_i^{(k-1)}
            =y_i-\widehat y_i^{(k-1)}$
            \State $\mathbf{T}_{1,i}^{(k)}
            =\textbf{SHAP}^{(k)}(x_{i\cdot})$
            \State $\mathbf{T}_{2,i}^{(k)}
            =\textbf{Q-SHAP}^{(k)}(x_{i\cdot})$
            by Algorithm~\ref{alg:qshap-general}
            \State $\boldsymbol{\Delta}_Q
            =\boldsymbol{\Delta}_Q
            +\alpha^2\mathbf{T}_{2,i}^{(k)}
            -2\alpha r_i^{(k-1)}
            \mathbf{T}_{1,i}^{(k)}$
        \EndFor
    \EndFor
    \State \textbf{return}
    $\textbf{RSQ-SHAP}
    =-\boldsymbol{\Delta}_Q/Q_{\emptyset}$
\end{algorithmic}
\end{algorithm}

Although the current software focuses on tree-based learners, the stagewise decomposition developed in this section is not specific to trees. It applies to any additive boosting model under squared-error loss, provided that the ordinary and quadratic Shapley terms can be evaluated for each weak learner.

\subsection{Accelerated algorithm for oblivious trees}

Oblivious trees have symmetric structures with all nodes at the same depth splitting on the same feature. This symmetry allows $T_1$ and $T_2$ to be calculated even more efficiently. We will exploit these properties here to further accelerate Algorithm~\ref{alg:qshap-general}.

Consider an oblivious tree of depth $D$. At depth $d=0,\ldots,D-1$, all nodes split using the same feature, denoted by $F[d]$. Thus, we define $F[d{:}]$ as the set of features used from depth $d$ to the leaves, that is 
\[
F[d{:}]=\bigcup_{s=d}^{D-1}F[s].
\]
For a non-terminal node $v$, let $v_L$ and $v_R$ denote the left and right children. Let $n_v$ be the number of samples reaching node $v$; hence $n_{v_L}$ and $n_{v_R}$ are the number of samples reaching the two children. For observation $x_{i\cdot}$, let
\[
c_d(x_{i\cdot})\in\{L,R\}
\]
denote the child selected by $x_{i\cdot}$ under the split rule at depth $d$. For a subset $F_d\subseteq F[d{:}]$, define
\[
F_{d+1} = F_d\cap F[d+1{:}].
\]

For a node $v$ at depth $d$, we calculate $\hat{m}_{F_d}(x_{i\cdot}, v)$ recursively by
\begin{equation}
\hat{m}_{F_d}(x_{i\cdot}, v)
=
\begin{cases}
\hat{m}^v,
& v \text{ is a leaf};\\[6pt]
\hat{m}_{F_{d+1}}(x_{i\cdot}, v_{c_d(x_{i\cdot})}),
& F[d]\in F_d;\\[8pt]
\dfrac{n_{v_L}}{n_v}\hat{m}_{F_{d+1}}(x_{i\cdot}, v_L)
+
\dfrac{n_{v_R}}{n_v}\hat{m}_{F_{d+1}}(x_{i\cdot}, v_R),
& F[d]\notin F_d.
\end{cases}
\label{eq:obl-recursion}
\end{equation}
Then, at the root node $v$, we have
\[
\hat{m}_{F_d}(x_{i\cdot})=\hat{m}_{F_d}(x_{i\cdot}, v).
\]
The recursion is evaluated with memoization over pairs $(v,F)$. Let $F_{\mathcal{T}}$ be the feature set used by a tree $\mathcal{T}$, then we have the following accelerated algorithm.

\begin{algorithm}[H]
\caption{\pkg{Q-SHAP} for oblivious trees}
\label{alg:qshap-obl}
\begin{algorithmic}

\State \textbf{\pkg{Q-SHAP}-OBL}($x_{i\cdot}$).

\State Compute $\hat{m}_F(x_{i\cdot})$ for all
$F \subseteq F_{\mathcal{T}}$
using recursion~\eqref{eq:obl-recursion}.

\For{$j \in F_{\mathcal{T}}$}

\State $\displaystyle
T_1[j]
=
\sum_{F \subseteq F_{\mathcal{T}}\setminus\{j\}}
\frac{
\hat{m}_{F\cup\{j\}}(x_{i\cdot})
-
\hat{m}_F(x_{i\cdot})
}{
|F_{\mathcal{T}}|
\binom{|F_{\mathcal{T}}|-1}{|F|}
}.
$

\State $\displaystyle
T_2[j]
=
\sum_{F \subseteq F_{\mathcal{T}}\setminus\{j\}}
\frac{
\hat{m}_{F\cup\{j\}}^2(x_{i\cdot})
-
\hat{m}_F^2(x_{i\cdot})
}{
|F_{\mathcal{T}}|
\binom{|F_{\mathcal{T}}|-1}{|F|}
}.
$

\EndFor

\State \Return
$T_1=\textbf{SHAP}(x_{i\cdot})$
and
$T_2=\textbf{Q-SHAP}(x_{i\cdot})$.

\end{algorithmic}
\end{algorithm}

Algorithm \ref{alg:qshap-obl} first computes the exact path-dependent subset prediction $\hat{m}_F(x_{i\cdot})$, and then applies the Shapley weights to both $\hat{m}_F(x_{i\cdot})$ and $\hat{m}_F^2(x_{i\cdot})$. The leaf-pair interactions used by the general Q-SHAP algorithm are not ignored and are included implicitly in $\hat{m}_F^2(x_{i\cdot})$. An oblivious tree has $2^d$ nodes at depth $d$ and at most $2^{D-d}$ unique subsets of $F[d{:}]$. Thus, Algorithm \ref{alg:qshap-obl} has at most $2^d2^{D-d}=2^D=L$ memoized states at each depth. The recursion over all depths costs $O(LD)$ in total, and the final Shapley sums also cost $O(LD)$. 

Further substantial speedups are possible by grouping observations according to the leaves they reach. 
We formalize the idea in the following theorem.

\begin{theorem}[Leaf Invariance of Q-SHAP for Oblivious Trees]
\label{thm:leaf-path-invariance}
Consider a fixed oblivious tree $\mathcal{T}$. Suppose two observations
$x_{i\cdot}$ and $x_{i'\cdot}$ reach the same leaf. Equivalently,
\[
c_d(x_{i\cdot})=c_d(x_{i'\cdot})
\qquad \text{for all } d=0,\ldots,D-1.
\]
Then, for every feature subset $F\subseteq F_{\mathcal{T}}$,
\[
\hat{m}_F(x_{i\cdot})=\hat{m}_F(x_{i'\cdot}).
\]
Consequently, the two observations have identical $T_1$ and $T_2$
values in Algorithm~\ref{alg:qshap-obl}:
\[
T_{1,i}[j]=T_{1,i'}[j],
\qquad
T_{2,i}[j]=T_{2,i'}[j],
\qquad j=1,\ldots,p.
\]
\end{theorem}

\begin{proof}
The recursion in Equation~\eqref{eq:obl-recursion} depends on the
observation only through the child index $c_d(x_{i\cdot})$ at each
depth $d$. If two observations reach the same leaf, then
$c_d(x_{i\cdot})=c_d(x_{i'\cdot})$ for all depths. Hence, for any
fixed subset $F$, the two observations choose the same child whenever
$F[d]\in F$ and use the same weighted average whenever $F[d]\notin F$.

Since the leaf values and the training proportions
$n_{v_L}/n_v$ and $n_{v_R}/n_v$ are properties of the tree and do not
depend on the observation, the recursive calculation of
$\hat{m}_F^v$ is identical at every node. In particular,
\[
\hat{m}_F(x_{i\cdot})=\hat{m}_F(x_{i'\cdot})
\]
for all $F\subseteq F_{\mathcal{T}}$.

The quantities $T_1$ and $T_2$ are Shapley sums of $\hat{m}_F$ and
$\hat{m}_F^2$, respectively. Since all subset predictions are
identical for the two observations, every summand in both Shapley sums
is identical. Therefore,
\[
T_{1,i}=T_{1,i'}
\qquad\text{and}\qquad
T_{2,i}=T_{2,i'}.
\]
\end{proof}

Note that for general trees, a subset of features can produce distinct paths along partial (hypothetical) trees, even for observations reaching the same final leaf, thereby yielding different conditional expectations. For an oblivious tree, however, the symmetric structure guarantees identical paths for any feature subset as long as two observations reach the same final leaf, significantly simplifying the computation. 

By grouping the observations by leaf, we need at most $\min(n, L)$ evaluations of Algorithm~\ref{alg:qshap-obl}. Therefore, for $K$ trees, the computational complexity is
\[
\boxed{
O\left(K\{nD+\min(n,L)LD\}\right)
}
\]
where the first term is of the same order as prediction. More importantly, once the sample size exceeds the number of leaves, the second term no longer increases with $n$. Compared with the complexity for general trees, $O(KnL^2D^2)$, Algorithm~\ref{alg:qshap-obl} provides a substantial computational improvement.

To illustrate this behavior, we generated synthetic regression data from
\[
Y_i
= 4X_{i_1} - 5X_{i_2} + 6X_{i_3} + 3X_{i_4} - X_{i_5}
+ \varepsilon_i,
\qquad
\varepsilon_i \sim N(0, 0.5^2),
\]
with all features independently generated from standard normal distributions. We fix the feature dimension at $p=100$ and vary the sample size $n \in \{1{,}000, 5{,}000, 10{,}000, 50{,}000, 100{,}000\}$ to evaluate the scalability of the oblivious tree implementation. For each setting, we train a CatBoost model with $K=100$ trees and otherwise use the default CatBoost parameters. All experiments were performed on a single core of an Apple M1 MacBook Pro with 16 GB of memory. Reported results are averaged over 10 replications.

Figure~\ref{catboost-time} shows that accelerated Q-SHAP is substantially faster than CatBoost model training itself across all the sample sizes considered. As the sample size increases from $1{,}000$ to $100{,}000$, the average runtime of accelerated Q-SHAP increases only from approximately $0.09$ seconds to $0.14$ seconds. This weak dependence on $n$ is consistent with the complexity analysis above. In comparison, model-training time increases steadily with sample size, while prediction time remains low. Figure~\ref{catboost-ablation} further shows that the accelerated Q-SHAP is \(26.5\) to \(590.9\) times as fast as the original Q-SHAP across different settings (sample size $n \in \{1,000,\ 5,000,\ 10,000\}$ and maximum tree depth $D \in \{2, 4, 6\}$), while the largest absolute difference in their reported
feature-specific \(R^2\) contributions is less than 
\(1.2 \times 10^{-15}\).

\begin{figure}[ht]
    \centering
    \includegraphics[width=\textwidth]{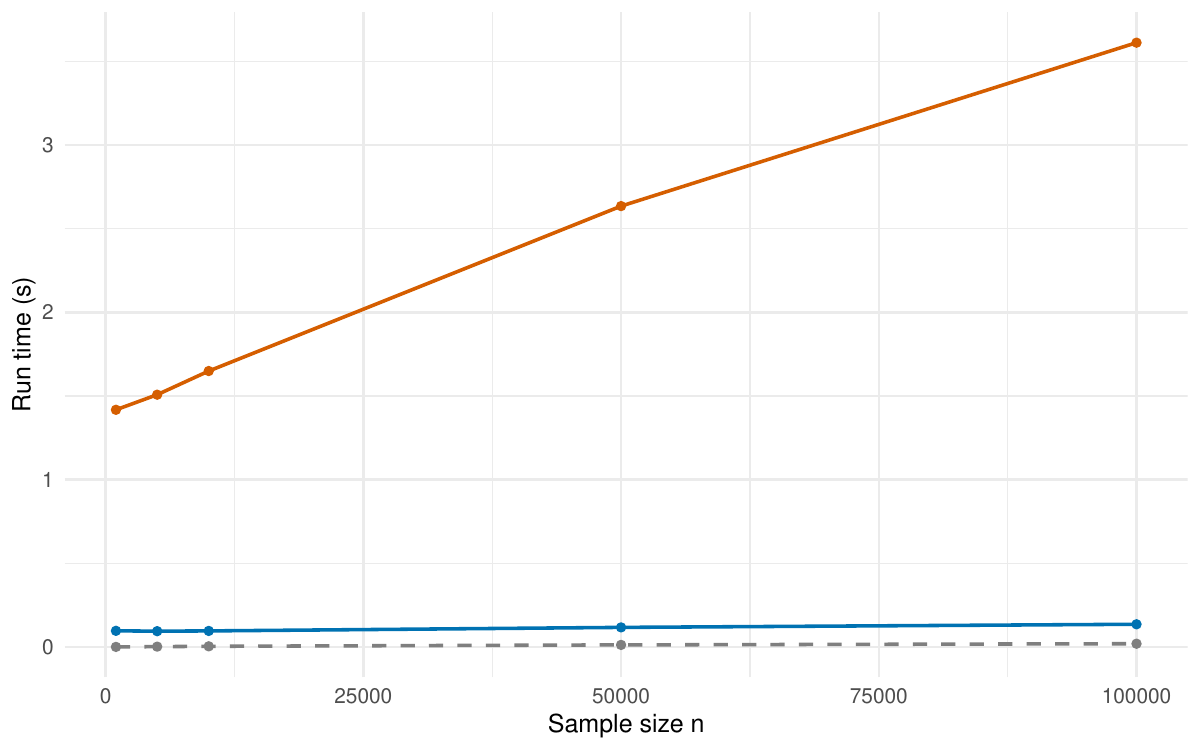}
\caption{Runtime scaling of CatBoost training (orange), accelerated Q-SHAP (blue), and Catboost prediction (gray) with increasing numbers of explained samples, averaged over 10 replications. The feature dimension is fixed at \(p=100\), and CatBoost is trained with \(K=100\) trees of depth 6. Prediction time is included for
reference.} \label{catboost-time}
\end{figure}

\begin{figure}[ht]
    \centering
    \includegraphics[width=\textwidth]{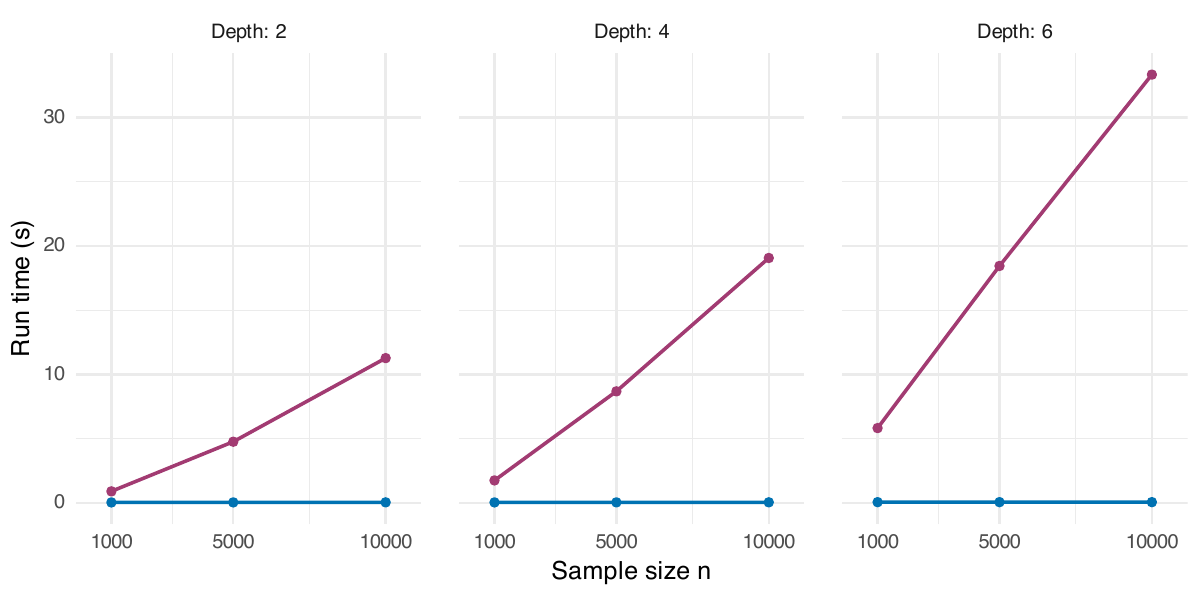}
\caption{Runtime scaling of original Q-SHAP (magenta) and accelerated Q-SHAP (blue) with increasing numbers of explained samples, averaged over 10 replications. The feature dimension is fixed at
\(p = 100\), and each CatBoost model contains \(K = 100\) trees. Panels correspond to different tree depths.}
\label{catboost-ablation}
\end{figure}

\FloatBarrier

\subsection{Generalized correlation coefficients} \label{sec:generalized_correlation}

For a subset of features $S$, its population coefficient of determination is
\[
R_S^2 = \frac{
\operatorname{Var}\{\mathbb{E}(Y\mid X_S)\}}{\operatorname{Var}(Y)}
=\operatorname{Cor}^2\{Y,m_S(X_S)\},
\]
following the fact that $\operatorname{Cov}\{Y,m_S(X_S)\} = \operatorname{Var}\{m_S(X_S)\}$. 
For a single feature $X_S$, $\sqrt{R_S^2}$ corresponds to the scale of the correlation coefficient between $Y$ and $m_S(X_S)$, as shown in simple linear regressions.

The oracle $R^2$ is monotone with respect to the available feature set.
Specifically, if $S\subseteq T$, then
\[
\mathbb{E}\left[\{Y-m_T(X_T)\}^2\right]
\le
\mathbb{E}\left[\{Y-m_S(X_S)\}^2\right],
\]
and hence $R_T^2\ge R_S^2$. Consequently, every marginal contribution
in the corresponding Shapley game is nonnegative, implying
\[
\phi_j^{R^2}\ge 0.
\]
This motivates the definition of generalized correlation coefficient for each feature $j$ as
\[
\rho_j^{\mathrm Q} = \sqrt{\phi_j^{R^2}}.
\]
The resulting quantities satisfy
\[
0\le \rho_j^{\mathrm Q}\le 1,
\qquad
\sum_{j\in \mathcal{P}}
\left(\rho_j^{\mathrm Q}\right)^2 = R_{\mathcal{P}}^2.
\]
Hence, $\rho_j^{\mathrm Q}$ quantifies the association between feature $j$ and the response, expressed on the familiar correlation scale. For any $S \subseteq \mathcal{P}$, $\sum_{j \in S} (\rho_j^Q)^2$ represents the portion of the full-model $R^2$ allocated to the features in $S$.

\section{The R package qshap} \label{sec:r_package}

The \pkg{qshap} package provides a user-friendly interface for fast calculation of feature-specific $R^2$ for gradient-boosted trees in \proglang{R}. The main workflow consists of two steps. The user first passes a supported fitted model to \code{gazer()}, which constructs a model-specific explainer, and then calls \code{rsq()} to compute the global feature-specific $R^2$ decomposition. The package currently supports models fitted with \pkg{xgboost}, \pkg{lightgbm}, and \pkg{catboost}. The package \pkg{qshap} is available on GitHub at
\href{https://github.com/catstats/Q-SHAP_R}{\url{https://github.com/catstats/Q-SHAP_R}} and can be directly installed from the Comprehensive R Archive Network (CRAN) for use.

\begin{CodeChunk}
\begin{CodeInput}
R> install.packages("qshap")
R> library("qshap")
\end{CodeInput}
\end{CodeChunk}

The remainder of this section describes the package interface and illustrates its use with the California housing data.

\subsection{Package overview}

Table~\ref{pkg-structure} summarizes the main user-facing functions. The \code{gazer()} function parses a fitted tree model and returns a \code{qshap_tree_explainer} object containing the model type, parsed trees, and cached tree summaries used by the \proglang{C++} backend. The \code{rsq()} function computes feature-specific $R^2$ values and returns a \code{qshap_result} object. The \code{loss()} function returns an individual-level local Shapley decomposition of loss contributions, while the standard \code{plot()} method provides several visual summaries of the resulting global importance scores. We present the basic workflow below: 

\begin{CodeChunk}
\begin{CodeInput}
R> model <- xgboost(x, y)
R> g <- gazer(model)
R> out <- rsq(g, x, y)
R> plot(out)
\end{CodeInput}
\end{CodeChunk}

Although the object returned by \code{gazer()} is model-specific,
all supported explainers share a single class,
\code{qshap\_tree\_explainer}. The \code{rsq()} and \code{loss()}
functions use the \code{model\_type} field to select the corresponding backend
while providing the same interface across models.

The object returned by \code{rsq()} has class \code{qshap_result}. It stores the feature-specific $R^2$ values, feature names, total $R^2$, number of features, as well as an optional local squared loss decomposition. The default print method reports the total $R^2$ and the top contributing features. The \code{summary()} methods return objects of classes
\code{summary.qshap\_tree\_explainer} and
\code{summary.qshap\_result}.
Their fields can be accessed programmatically, and the corresponding
\code{print()} methods output the formatted summaries.

\begin{table}[ht]
\centering
\begin{tabular}{lll}
\toprule
Goal & Function & Return value \\
\midrule
Build explainer & \code{gazer()} & \code{qshap_tree_explainer} \\
Feature-specific $R^2$ & \code{rsq()} & \code{qshap_result} \\
Local loss decomposition & \code{loss()} & loss matrix \\
Visualization & \code{plot()} & S3 plot for \code{qshap_result} \\
\bottomrule
\end{tabular}
\caption{Main user-facing functions in the \pkg{qshap} package.}
\label{pkg-structure}
\end{table}

\subsection{A case study with California housing data}

The California housing data, originally derived from the 1990 U.S. Census, are widely used for regression benchmarks. The dataset contains 20,640 observations describing demographic and housing characteristics of California census block groups. Each observation includes nine predictor variables: longitude, latitude, median housing age, total number of rooms and bedrooms, block-group population, number of households, median income, and ocean proximity. We fit an \pkg{xgboost} model to predict the median house price, and use Q-SHAP to decompose the $R^2$ for feature-specific interpretation.

\subsection{Data loading and processing}

We download the data directly from \pkg{OpenML} \citep{bischl2025openml}, and convert it to a data frame for subsequent analysis. We use \code{median\_house\_value} as the response and the remaining variables as predictors. The categorical variable \code{ocean_proximity} is converted to numeric codes before model fitting for ease of interpretation.

\begin{CodeChunk}
\begin{CodeInput}
R> library("OpenML")
R> oml_data <- getOMLDataSet(data.id = 43939)
R> cal <- oml_data[["data"]]
R> cal <- as.data.frame(cal)

R> cal[["ocean_proximity"]] <- as.numeric(cal[["ocean_proximity"]],
                               levels = sort(unique(cal[["ocean_proximity"]])))
R> target_col <- "median_house_value"
R> y <- cal[[target_col]]
R> X <- as.data.frame(cal[, setdiff(colnames(cal), target_col), drop = FALSE])
R> feature_names <- colnames(X)
\end{CodeInput}
\end{CodeChunk}

\subsection{Model fitting}

To keep the example focused on the Q-SHAP workflow, we deliberately fit a simple \pkg{xgboost} model with 50 boosting rounds and a maximum tree depth of two. All other parameters retain their default values.  We set the random seed to ensure reproducibility.
\begin{CodeChunk}
\begin{CodeInput}
R> set.seed(42)
R> library("xgboost")
R> model <- xgboost(x = X, y = y, nrounds = 50, max_depth = 2)
\end{CodeInput}
\end{CodeChunk}

\subsection{Constructing the explainer} \label{sec:explainer}

After fitting the model, we can directly pass the fitted model to the \code{gazer()} function of \pkg{qshap}, which will extract the essential tree structure to compute Q-SHAP. The \code{gazer()} function converts backend-specific models from \pkg{xgboost}, \pkg{lightgbm}, and \pkg{catboost} into a unified structure for subsequent computations.
 
\begin{CodeChunk}
\begin{CodeInput}
R> g <- gazer(model)
\end{CodeInput}
\end{CodeChunk}

To illustrate this representation, we return the first tree through \code{get_tree()}. The returned \code{simple\_tree} object contains 10 fields: \code{children_left} and \code{children_right} encode the left and right child-node indices; \code{feature} and \code{threshold} define the features and threshold split at each internal node; \code{max_depth} records the maximum tree depth; \code{n_node_samples} stores the number of samples at each node of the tree; \code{value} contains the node values; \code{node_count} records the total number of nodes in the tree; the fields \code{default_left} and \code{xgboost_split} preserve backend-specific routing conventions. Any binary decision tree represented in this format can be processed by \pkg{qshap} for the downstream calculation of feature-specific $R^2$; this structure is also similar to \pkg{scikit-learn}, providing a basis for adding further model backends.

\begin{CodeChunk}
\begin{CodeInput}
R> tree <- get_tree(g, 1)
R> print(unclass(tree))
\end{CodeInput}
\end{CodeChunk}

\begin{CodeOutput}
$children_left
[1]  1  3  5 -1 -1 -1 -1

$children_right
[1]  2  4  6 -1 -1 -1 -1

$feature
[1] 7 7 7 0 0 0 0

$threshold
[1]      5.0329      3.1076      6.8220 -21061.1480    840.2939
[6]  25087.2460  64406.3800

$max_depth
[1] 2

$n_node_samples
[1] 20640 16249  4391  8053  8196  3056  1335

$value
[1]  4.408701e-03 -3.338228e+04  1.235114e+05 -2.106115e+04
[5]  8.402939e+02  2.508725e+04  6.440638e+04

$node_count
[1] 7

$default_left
[1] FALSE FALSE FALSE FALSE FALSE FALSE FALSE

$xgboost_split
[1] TRUE
\end{CodeOutput}

\subsection{Calculation of feature-specific \texorpdfstring{$R^2$}{R-squared} values} \label{rsq_example}

Once the explainer has been constructed using \code{gazer()}, we use the \code{rsq()} function to compute the feature-specific $R^2$ decomposition. The returned \code{qshap_result} object contains the empirical $R^2$ of the fitted model and the contribution assigned to each feature. Its \code{print} method reports the feature contributions in decreasing order.

\begin{CodeChunk}
\begin{CodeInput}
R> phi_rsq_R <- rsq(g, X, y, feature_names = feature_names, local = TRUE)
R> print(phi_rsq_R)
\end{CodeInput}
\end{CodeChunk}

\begin{CodeOutput}
<qshap_result>
  Total R^2: 0.7646 
  Number of features: 9 
  Number of samples: 20640 
  Raw loss contributions: 20640 x 9 
  Local R^2-scale contributions: 20640 x 9 

Top 9 features by R^2:
            Feature R_squared
      median_income  0.465827
    ocean_proximity  0.115210
          longitude  0.067137
           latitude  0.059218
     total_bedrooms  0.018957
 housing_median_age  0.016743
         population  0.015394
         households  0.003868
        total_rooms  0.002241
\end{CodeOutput}

Q-SHAP computes the decomposition exactly from the fitted tree structure, without Monte Carlo sampling. By the Shapley efficiency property, the feature-specific contributions sum to the empirical $R^2$ of the fitted model. The following calculation verifies this identity: 
\begin{CodeChunk}
\begin{CodeInput}
R> result_summary <- summary(phi_rsq_R)
R> model_rsq <- 1 - sum((y - predict(model, X))^2) / sum((y - mean(y))^2)
R> print(c(sum_feature_rsq = result_summary[["total_rsq"]],
+          fitted_model_rsq = model_rsq))
\end{CodeInput}
\end{CodeChunk}

\begin{CodeOutput}
 sum_feature_rsq fitted_model_rsq 
       0.7645957        0.7645957
\end{CodeOutput}

\subsection{Local decomposition of the squared loss}

Besides the global decomposition, \pkg{qshap} can return observation-level squared loss contributions from which the feature-specific $R^2$ values are obtained. For each observation, summing the contributions recovers the change in squared loss after fitting the model. Users interested only in the local decomposition can obtain it directly with \code{loss()}, and it can be simply returned by setting \code{local = TRUE} when calling \code{rsq()} as well. 
The \code{local\_rsq()} accessor retrieves the observation-level
contributions on the $R^2$ scale, stored as the squared-loss
matrix scaled by $-1/Q_{\emptyset}$. Contributions in column $j$ then sum up to the corresponding global feature-specific $R^2$ contribution $\hat{\phi}_j^{R^2}$. In the local decomposition, a positive value indicates that a feature contributes positively to the model fit for a given observation, whereas a negative value indicates a contribution in the opposite direction. 

\begin{CodeChunk}
\begin{CodeInput} 
R> local_rsq_R <- local_rsq(phi_rsq_R)
R> head(local_rsq_R)
\end{CodeInput}
\end{CodeChunk}

\begin{CodeOutput}
         longitude      latitude housing_median_age   total_rooms
[1,]  2.313650e-06 -2.104842e-06       4.155820e-06 -3.214032e-08
[2,] -1.080987e-05  1.426728e-05       1.177522e-06 -1.055277e-06
[3,] -1.198699e-05  1.633523e-05      -9.178597e-06  6.009082e-07
[4,]  6.762917e-08  5.414373e-07       5.342529e-06  1.659817e-08
[5,]  1.439236e-05 -1.806204e-05       1.939579e-05 -8.407056e-07
[6,]  2.970003e-06 -3.652318e-06       6.376952e-06 -1.313026e-07
     total_bedrooms    population    households median_income
[1,]  -1.903035e-06  2.062535e-06 -2.926779e-07  2.039209e-04
[2,]  -7.934363e-06  4.557886e-06 -5.312390e-06  7.687019e-05
[3,]   1.398375e-05 -1.375490e-05  2.249859e-06  6.815944e-05
[4,]  -1.475455e-07  3.974378e-08  3.493709e-08  5.078222e-05
[5,]  -7.238168e-06  2.120278e-05 -2.484070e-06 -4.423180e-06
[6,]  -3.233182e-06  3.556497e-06 -5.012750e-07 -2.587918e-06
     ocean_proximity
[1,]    1.032448e-05
[2,]   -4.260958e-06
[3,]   -2.708453e-06
[4,]    9.010335e-06
[5,]    2.305173e-05
[6,]    9.185553e-06
\end{CodeOutput}

\subsection{Parallel computing}

The Q-SHAP decomposition is additive across observations, allowing individual observations to be processed independently. The \proglang{R} implementation therefore adopts a divide-and-combine strategy that partitions the dataset into disjoint subsets, processes them in parallel on separate workers, and aggregates the results to obtain the global feature-specific estimates.

In the R implementation, \code{rsq()} supports multicore execution through the \pkg{parallel} package using a PSOCK cluster (\code{parallel::makeCluster()}), which is compatible with CRAN environments. The \code{ncore} argument controls the number of workers, with \code{ncore = 1}  (the default) for serial execution and \code{ncore = -1} for using all available cores. Parallel computing is most beneficial for large datasets, large ensembles, or deep trees, where the computational workload outweighs the overhead associated with worker startup and communication.

For example, the previous example can be evaluated using eight cores as follows:
\begin{CodeChunk}
\begin{CodeInput}
R> phi_rsq_R_parallel <- rsq(g, X, y, feature_names = feature_names, ncore = 8)
\end{CodeInput}
\end{CodeChunk}

\subsection{Visualization}

Visualization is seamlessly integrated into the \pkg{qshap} workflow through a unified S3 method \code{plot()}, implemented with \pkg{ggplot2} \citep{ggplot2}. The method supports a variety of visualization techniques for global and local feature-specific $R^2$, including bar plots, elbow plots, cumulative contribution plots, and heatmaps. The bar plot is the default, while setting \code{type = "gcorr"} displays the bar plot for feature-specific generalized correlation coefficients. When \code{save_name} is supplied, the resulting plot is saved as a PDF file. Whereas the package \pkg{shapviz} \citep{mayer2024shapviz} provides visualization tools for local SHAP and their aggregations, the plotting functionality in \pkg{qshap} is designed specifically for global and observation-level $R^2$ decomposition.

The following code produces the bar plots for feature-specific $R^2$ and corresponding generalized correlation coefficients for the California housing data:
\begin{CodeChunk}
\begin{CodeInput}
R> library("ggplot2")
R> plot(phi_rsq_R, label = feature_names, rotation = 45,
+       save_name = "california_housing_bar")
R> plot(phi_rsq_R, type = "gcorr", label = feature_names, rotation = 45, 
+       save_name = "california_housing_gcorr")
\end{CodeInput}
\end{CodeChunk}

\begin{figure}[!htbp]
    \centering
        \includegraphics[width=\textwidth]{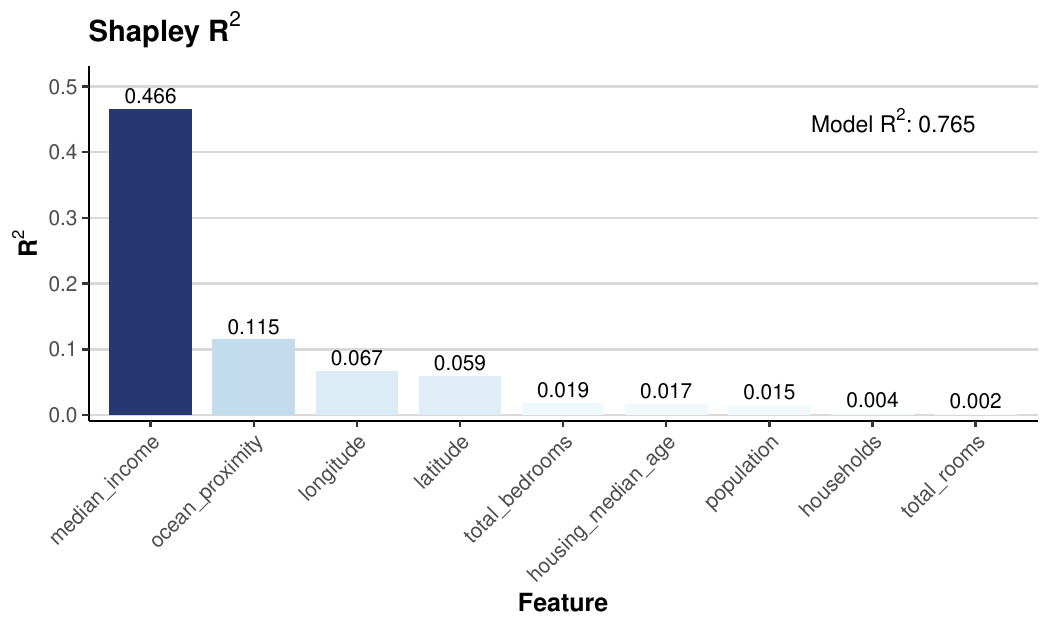}
    \caption{Feature-specific \(R^2\) values for the
     California housing data.}
    \label{fig:california_housing_all}
\end{figure}

\begin{figure}[!htbp]
    \centering

        \includegraphics[width=\textwidth]{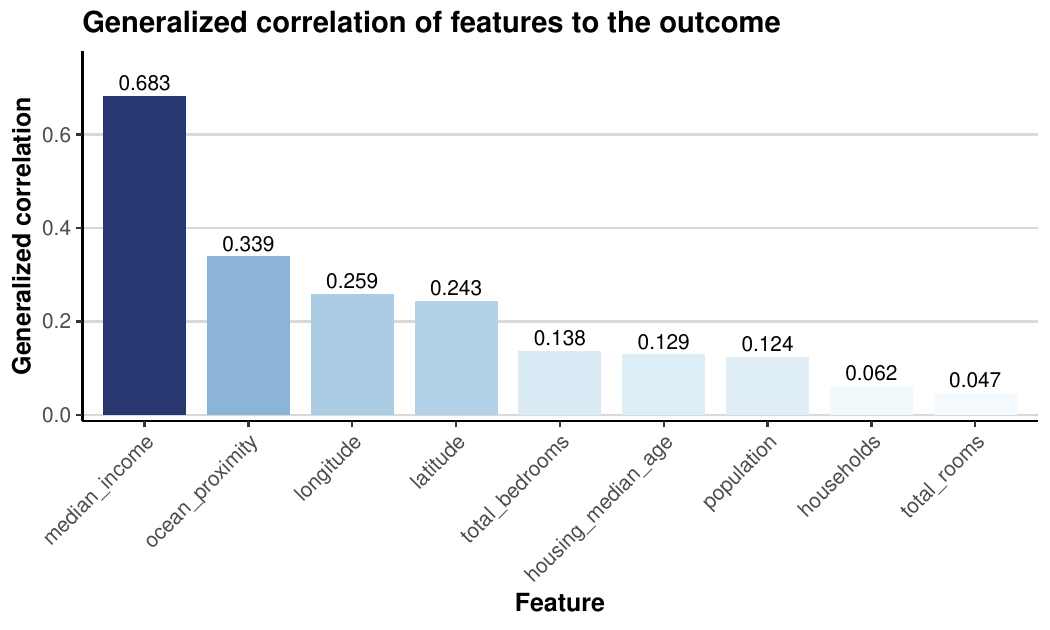}
    \caption{Feature-specific generalized correlation coefficients for the California housing data.}
    \label{fig:california_housing_gcorr}
\end{figure}

Figures \ref{fig:california_housing_all} and \ref{fig:california_housing_gcorr} show that median income has the largest contribution, followed by ocean proximity, longitude, and latitude. The generalized correlation coefficient plot preserves the feature rankings while expressing the contributions on a correlation-like scale.

Beyond the individual feature displays, \pkg{qshap} provides elbow and cumulative contribution plots that summarize the ranked feature contributions. The elbow plot (Figure~\ref{fig:california_housing_elbow}) highlights changes in contribution and can be used to identify a core set of leading features, whereas the cumulative plot (Figure~\ref{fig:california_housing_cumu}) shows how the contributions accumulate toward the fitted model's total $R^2$.

\begin{CodeChunk}
\begin{CodeInput}
R> plot(phi_rsq_R, type = "elbow", label = feature_names, rotation = 45,
+       save_name = "california_housing_elbow")
R> plot(phi_rsq_R, type = "cumu", max_comp = 9, label = feature_names,
+       save_name = "california_housing_cumu")
\end{CodeInput}
\end{CodeChunk}

\begin{figure}[!htbp]
    \centering
        \centering
        \includegraphics[width=\textwidth]{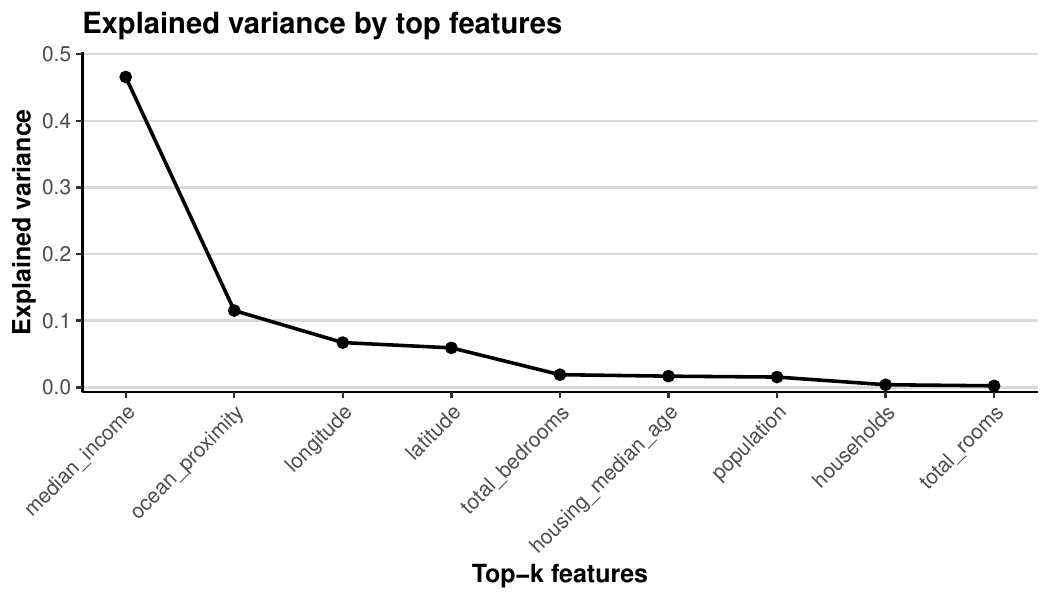}
   \caption{Elbow plot of feature-specific \(R^2\) values
    for the California housing data.}
    \label{fig:california_housing_elbow}
\end{figure}

\begin{figure}[!htbp]
    \centering
        \centering
        \includegraphics[width=\textwidth]{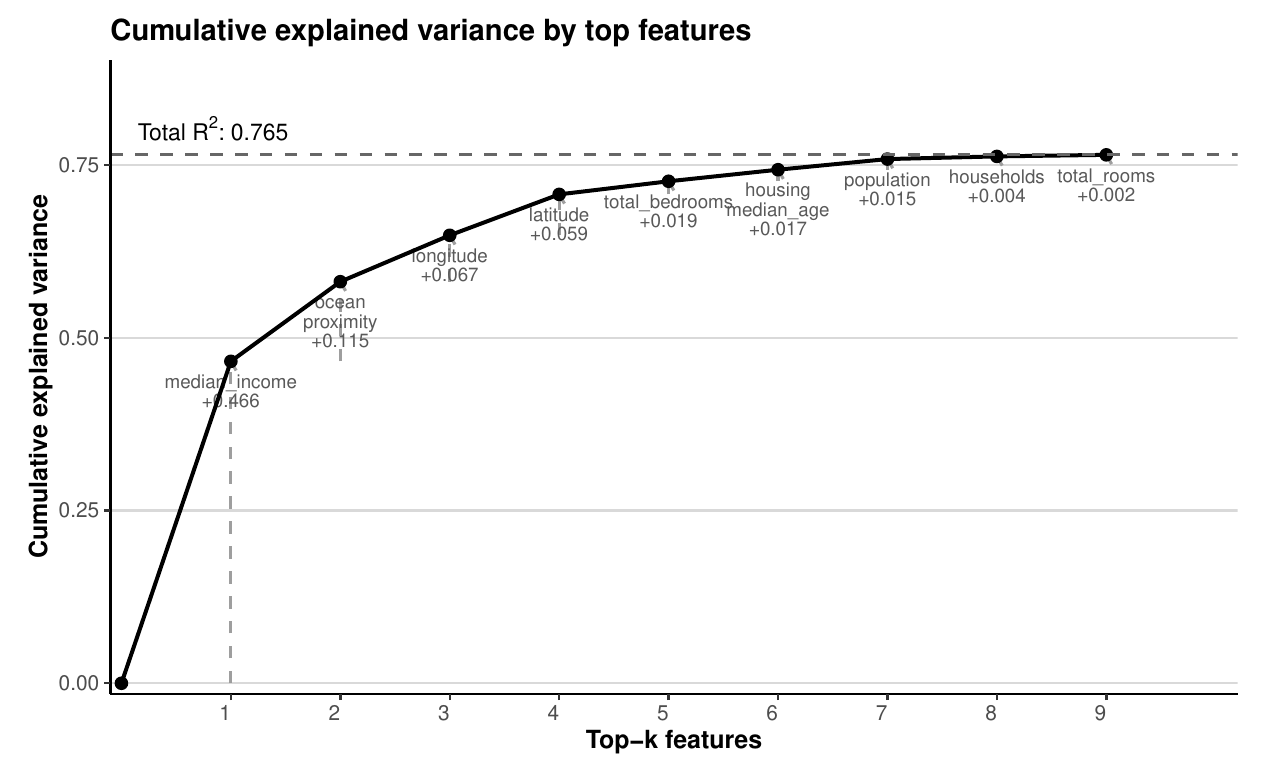}
    \caption{Cumulative feature-specific \(R^2\) values of
    the top-ranked features for the California housing data.}
    \label{fig:california_housing_cumu}
\end{figure}

We also provide a heatmap for inspecting observation-level contributions on the $R^2$ scale. Users may specify the observations to display in the heatmap. The heatmap (Figure~\ref{fig:california_housing_heatmap}) highlights observations with the most extreme total contributions by default, with rows ordered by their signed row totals. The argument \code{n_show} controls the number of observations displayed, reporting observations with the largest positive and most negative total contributions to the global $R^2$ decomposition.

\begin{CodeChunk}
\begin{CodeInput}
R> plot(phi_rsq_R, type = "heatmap", feature_names = feature_names, n_show = 30,
+       save_name = "california_housing_heatmap")
\end{CodeInput}
\end{CodeChunk}

\begin{figure}[!htbp]
    \centering
        \centering
        \includegraphics[width=\textwidth]{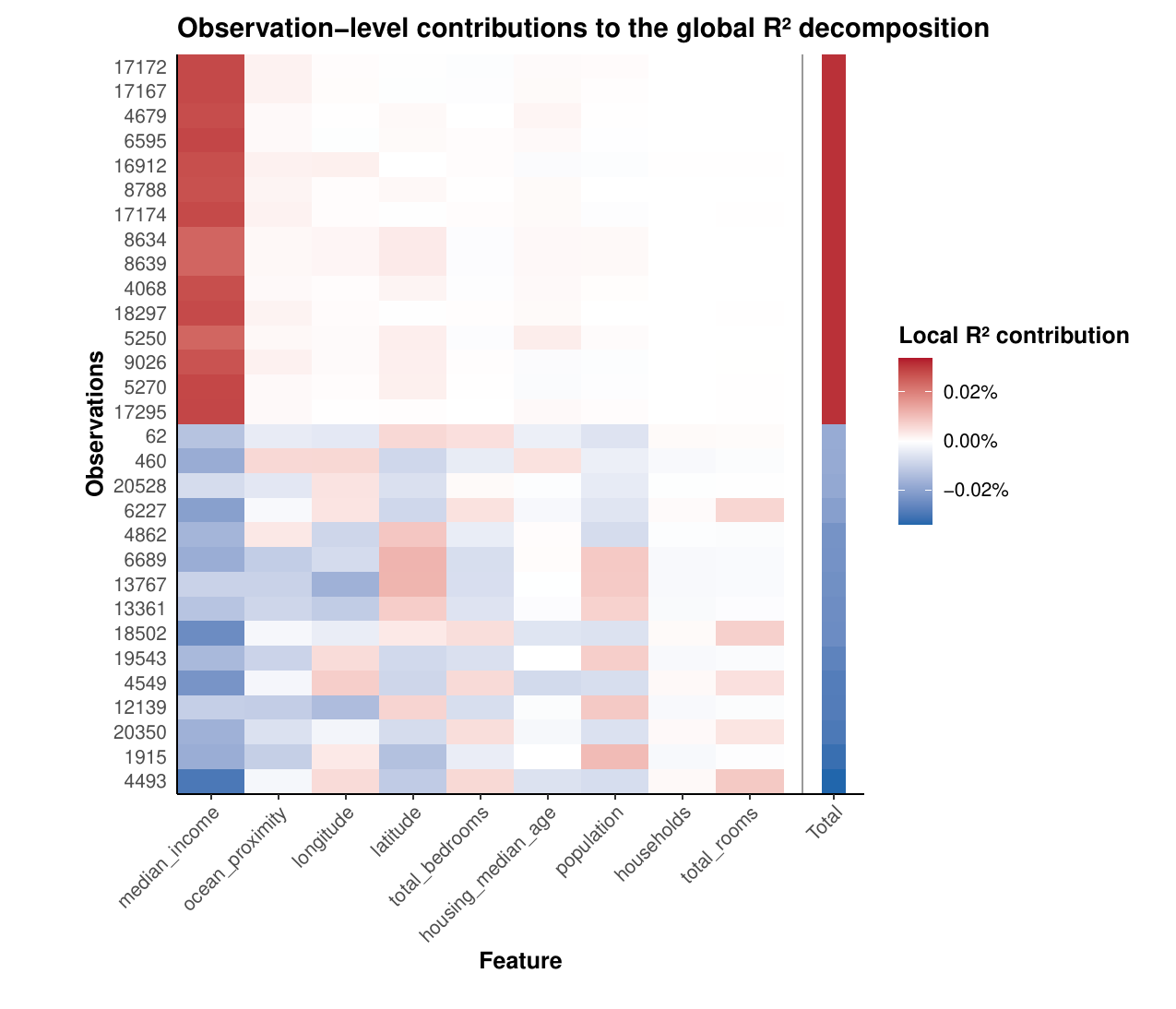}
    \caption{Observation-level contributions to the global \(R^2\) decomposition for selected observations.}
    \label{fig:california_housing_heatmap}
\end{figure}

\FloatBarrier
\section{The Python package qshap} \label{sec:python_package}

The \proglang{Python} implementation of \pkg{qshap} follows the same conceptual workflow as the \proglang{R} package through a \proglang{Python} interface to the same \proglang{C++} backend.
A fitted model is first passed to \code{gazer()}, which converts the structure to a unified internal format; the resulting structure is then supplied to \code{rsq()} to compute the feature-specific $R^2$ values. The \proglang{Python} interface accepts and returns \pkg{NumPy}-based numerical objects and provides \pkg{Matplotlib}-based functions for visualization \citep{harris2020array, Hunter:2007}.   

The package currently supports decision-tree and gradient-boosting estimators from \pkg{scikit-learn} \citep{pedregosa2011scikit}, as well as \pkg{xgboost}, \pkg{lightgbm}, and \pkg{catboost}.
It is publicly available via GitHub and Python Package Index (PyPI). 

Since Section \ref{sec:r_package} provides a complete workflow using the \proglang{R} package, this section focuses on the \proglang{Python}-specific interface and demonstrates its usage via \pkg{xgboost} on the California housing data.

\subsection{A case study with California housing data}

We revisit the California housing example from
Section~\ref{sec:r_package} to demonstrate model fitting with
\pkg{xgboost}, computation of Q-SHAP contributions, and visualization of
the results in \proglang{Python}.

\begin{CodeChunk}
\begin{CodeInput}
>>> import os
>>> import numpy as np
>>> import pandas as pd
>>> import xgboost as xgb
>>> from qshap import gazer, vis


>>> out_dir = "qshap_california_artifacts"

>>> X = pd.read_csv(os.path.join(out_dir, "X.csv"))
>>> y = pd.read_csv(os.path.join(out_dir, "y.csv"))["y"]

>>> feature_names = X.columns.to_numpy()
>>> X_np = X.to_numpy()
>>> y_np = y.to_numpy()


>>> model = xgb.XGBRegressor(max_depth = 2, n_estimators = 50)
>>> model.fit(X_np, y_np)


>>> g = gazer(model)
>>> local_result = g.rsq(X_np, y_np, local = True, ncore = 8)
>>> phi_rsq = local_result.rsq
>>> local_loss = local_result.loss


>>> order = np.argsort(phi_rsq)[::-1]
>>> results = pd.DataFrame({
...     "Feature": feature_names[order],
...     "R_squared": phi_rsq[order]
...     })
>>> results


>>> vis.rsq(
...     phi_rsq,
...     label = np.array(feature_names),
...     rotation = 30,
...     save_name = "california_housing_python",
...     color_map_name = "Pastel2"
... )

\end{CodeInput}
\begin{CodeOutput}
median_income 0.465827
ocean_proximity 0.115210
longitude  0.067137
latitude   0.059218
total_bedrooms 0.018957
housing_median_age 0.016743
population 0.015394
households 0.003868
total_rooms 0.002241
\end{CodeOutput}
\end{CodeChunk}
Using the model fitted under the same parameters and the same data, the
\proglang{R} and \proglang{Python} implementations return consistent feature-specific \(R^2\) values at the reported
precision.

\section{Internal design of the R package} \label{sec:internal}

\subsection{Model support and dispatch}

The package interface of \pkg{qshap} consists of three main functions. \code{gazer()} constructs an explainer from a fitted model; \code{rsq()} computes the feature-specific $R^2$ values; and \code{plot()} visualizes the results. Internally, \code{gazer()} is an S3 generic with separate methods for \pkg{xgboost}, \pkg{lightgbm}, and \pkg{catboost}. Each method parses the native model representation and converts it into the common tree objects used by the numerical backend. This separation keeps model-specific prediction, split-routing, node-cover, and initialization conventions within the parser, while the downstream Q-SHAP calculation operates on a common representation. The main classes are summarized in Table~\ref{tab:qshap-classes}.

\subsection{Class design} \label{sec:class-design}

The \proglang{R} implementation uses lightweight S3 classes. Each object
is implemented as a list with a class attribute, created through a dedicated constructor, and checked by a validation function. User-facing classes add print, summary, and plotting methods when applicable.

A \code{simple_tree} stores the structure of a single binary tree in a model-independent format. Its fields include left and right child indices, split feature indices, split thresholds, maximum depth, node sample counts or weights, node values, and total node count. When required, the object also records learned missing-value directions and backend-specific split conventions. Leaf nodes are represented by child index \(-1\). This object is close to the representation used by common tree libraries and serves as the main interchange format inside the package. 
A \code{tree_summary} is a computation-ready version of a \code{simple_tree}. In addition to the child, feature, threshold, and depth information, it stores the set of unique split features used by the tree, node-specific sample weights, and initial prediction values used by the Q-SHAP backend. These derived quantities are reused during the calculation of \(T_1\), \(T_2\), and feature-specific loss contributions. The \code{tree_summary} object keeps numerical preparation separate from \code{simple_tree}'s model-parsing handling. The former caches quantities needed by the Q-SHAP algorithm that may not be directly available from the native model object. 

The object returned by \code{gazer()} has class \code{qshap\_tree\_explainer}. It stores the original model object, the model type, the parsed list of trees, the maximum depth, the base score when applicable, cached tree summaries, and precomputed numerical arrays used by the \proglang{C++} backend. The explainer therefore connects a fitted model and the Q-SHAP computation.

The main user-facing class is \code{qshap_result}. It stores the feature-specific \(R^2\) vector, feature names, total \(R^2\), sample size, number of features, and optionally the local loss contribution matrix. Its methods allow users to print, summarize, plot, or convert the result to a data frame.

\begin{table}[htbp]
\centering
\begin{tabular}{ll}
\hline
Class & Role \\
\hline
\code{simple_tree} & Model-independent representation of one parsed tree. \\
\code{tree_summary} & Computation-ready tree summary used by the backend. \\
\code{qshap_tree_explainer} & Parsed model object returned by \code{gazer()}. \\
\code{qshap_result} & User-facing feature-specific \(R^2\) result. \\

\hline
\end{tabular}
\caption{Main S3 classes used by \pkg{qshap}.}
\label{tab:qshap-classes}
\end{table}

\subsection{Backend architecture} \label{sec:backend-architecture}

A fitted model is first passed to \code{gazer()}, which dispatches on the model class. The corresponding formatter converts the model-specific representation into a list of \code{simple_tree} objects. Each \code{simple_tree} is then converted to a \code{tree_summary}. This separation keeps model-library conventions outside the numerical
\pkg{qshap} backends.

For each tree, backend-specific routines obtain the ordinary per-tree TreeSHAP term
$T_1$, either from the native model library or from a compatible
TreeSHAP implementation. The compiled numerical kernels use the cached
tree summaries to compute the quadratic term $T_2$ by Q-SHAP and combine
$T_1$, $T_2$, and the stagewise residuals to obtain the local loss
contributions and feature-specific $R^2$ values.

The internal function \code{qshap_loss()} dispatches on the stored model
type and selects the corresponding general or specialized computational
routine without changing the user-facing interface. The main numerical
operations are implemented in compiled \proglang{C++} code through
\pkg{Rcpp}. Complex roots and inverse coefficients required by the
general-tree algorithm are precomputed once and stored in the explainer,
so that repeated calls to \code{rsq()} do not recalculate them.

\subsection{Model-specific formatters} \label{sec:formatters}

The formatter layer translates a fitted model from its native representation into the common \code{simple_tree} objects used by the numerical backend. It is responsible for preserving all library-specific information that affects prediction, including tree topology, split thresholds, missing-value routing, node weights, initial predictions, and output scaling.

For \pkg{xgboost}, the formatter reads the JSON representation of the fitted booster, extracts the base score, and parses the tree topology, split conditions, learned missing-value directions, node cover, and leaf
outputs. For \pkg{lightgbm}, the formatter extracts the corresponding
information from the fitted model representation and constructs the same
internal tree objects.

For \pkg{catboost}, the formatter reads the JSON representation and extracts the model scale and bias. Non-symmetric trees are handled through the general-tree representation, whereas symmetric, or oblivious, trees can
be represented as complete binary trees and passed to the specialized
oblivious-tree implementation. For a symmetric tree of depth $D$, the
complete representation contains $2^{D+1}-1$ nodes; internal node weights
and values are reconstructed from the leaf weights and outputs.

\subsection{Extending to new tree models} \label{sec:adding-models}

The Q-SHAP algorithm applies to arbitrary boosted ensembles of binary decision trees; extending the \pkg{qshap} software to a new
model class additionally requires recovery of the fitted tree
structure and the ordinary per-tree TreeSHAP quantities.  At the package-development level, support for a new model class requires an S3 method for \code{gazer()}, a model-specific formatter, and corresponding updates in the validation and loss-dispatch layers. The formatter converts the native model representation into validated \code{simple_tree} objects.

The common tree representation supplies the quantities required by the
general \pkg{qshap} backend to compute the quadratic term $T_2$. The adapter must additionally provide the ordinary per-tree TreeSHAP term $T_1$, either through the native model library, through an existing TreeSHAP implementation that accepts the parsed trees, or through a model-specific routine. Once $T_1$ and the corresponding \code{tree_summary} objects are available, the existing loss decomposition and feature-specific $R^2$ calculation can be reused without modifying the user-facing workflow.

For a general binary-tree boosting ensemble, the existing \pkg{qshap} backend can be reused. If a model has additional structure, its adapter may instead use a specialized computational routine. The oblivious-tree implementation provides one example to reuse computations for observations assigned to the same leaf without changing the user-facing interface.

\section{Conclusion} \label{sec:conclusion}

Gradient-boosted tree ensembles are widely used because of their effectiveness in tabular prediction tasks, but their increasing structural complexity and large number of trees make them difficult to interpret. Local methods based on Shapley values have been widely used to explain individual predictions, yet it remains challenging to understand how individual features contribute to the overall model fit. In this paper, we introduced \pkg{qshap}, software for computing
feature-specific \(R^2\) values for boosted tree models in both \proglang{R}
and \proglang{Python}. The package implements Q-SHAP, an exact algorithm for computing Shapley-value
decomposition of $R^2$. It provides a compact workflow for
constructing model explainers, extracting observation-level loss contributions, calculating global contributions, and visualizing feature importance.

The package supports widely used GBDT implementations, including \pkg{xgboost}, \pkg{lightgbm}, and \pkg{catboost}, through a shared internal tree representation and efficient compiled backends in \proglang{C++}. For arbitrary binary trees, \pkg{qshap} uses the stable general-tree Q-SHAP algorithm. For oblivious trees, with CatBoost as the main example, we introduced a specialized backend that exploits the symmetric tree structure, groups observations by leaf, and reuses the corresponding computations, substantially reducing computation time.

The current work focuses on the decomposition of a well-defined $R^2$ under squared-error loss. It also provides a foundation for extending the framework to more general loss functions.

\section*{Acknowledgments}

We thank Steven He for his help with the initial development of the \proglang{C++} code and its integration with \proglang{R}. This research was partially supported by NIH grants R01GM131491, R01AG080917, and R01AG080917-02S1, NCI grants R03 CA235363 and P30CA062203, and UCI Anti-Cancer Challenge funds from the UC Irvine Comprehensive Cancer Center. The content is solely the responsibility of the authors and does not necessarily represent the official views of the National Institutes of Health or the Chao Family Comprehensive Cancer Center.


\bibliography{refs}

\begin{thebibliography}{30}
\newcommand{\enquote}[1]{``#1''}
\providecommand{\natexlab}[1]{#1}
\providecommand{\url}[1]{\texttt{#1}}
\providecommand{\urlprefix}{URL }
\expandafter\ifx\csname urlstyle\endcsname\relax
  \providecommand{\doi}[1]{doi:\discretionary{}{}{}#1}\else
  \providecommand{\doi}{doi:\discretionary{}{}{}\begingroup
  \urlstyle{rm}\Url}\fi
\providecommand{\eprint}[2][]{\url{#2}}

\bibitem[{B{\'e}nard \emph{et~al.}(2022)B{\'e}nard, Biau, Da~Veiga, and
  Scornet}]{benard2022shaff}
B{\'e}nard C, Biau G, Da~Veiga S, Scornet E (2022).
\newblock \enquote{{SHAFF: Fast and Consistent SHApley eFfect Estimates via
  Random Forests}.}
\newblock In \emph{International Conference on Artificial Intelligence and
  Statistics}, pp. 5563--5582. PMLR.

\bibitem[{Bifet \emph{et~al.}(2022)Bifet, Read, Xu
  \emph{et~al.}}]{bifet2022linear}
Bifet A, Read J, Xu C, \emph{et~al.} (2022).
\newblock \enquote{{Linear TreeShap}.}
\newblock \emph{Advances in Neural Information Processing Systems},
  \textbf{35}, 25818--25828.

\bibitem[{Bischl \emph{et~al.}(2025)Bischl, Casalicchio, Das, Feurer, Fischer,
  Gijsbers, Mukherjee, M{\"u}ller, N{\'e}meth, Oala
  \emph{et~al.}}]{bischl2025openml}
Bischl B, Casalicchio G, Das T, Feurer M, Fischer S, Gijsbers P, Mukherjee S,
  M{\"u}ller AC, N{\'e}meth L, Oala L, \emph{et~al.} (2025).
\newblock \enquote{OpenML: Insights from 10 Years and More Than a Thousand
  Papers.}
\newblock \emph{Patterns}, \textbf{6}(7).

\bibitem[{Breiman(2001)}]{breiman2001random}
Breiman L (2001).
\newblock \enquote{{Random Forests}.}
\newblock \emph{Machine Learning}, \textbf{45}, 5--32.

\bibitem[{Chen and Guestrin(2016)}]{chen2016xgboost}
Chen T, Guestrin C (2016).
\newblock \enquote{{XGBoost: A Scalable Tree Boosting System}.}
\newblock In \emph{Proceedings of the 22nd ACM SIGKDD International Conference
  on Knowledge Discovery and Data Mining}, pp. 785--794.

\bibitem[{Covert \emph{et~al.}(2020)Covert, Lundberg, and
  Lee}]{covert2020understanding}
Covert I, Lundberg SM, Lee SI (2020).
\newblock \enquote{Understanding Global Feature Contributions with Additive
  Importance Measures.}
\newblock \emph{Advances in Neural Information Processing Systems},
  \textbf{33}, 17212--17223.

\bibitem[{Fisher \emph{et~al.}(2019)Fisher, Rudin, and
  Dominici}]{fisher2019all}
Fisher A, Rudin C, Dominici F (2019).
\newblock \enquote{All Models Are Wrong, but Many Are Useful: Learning a
  Variable's Importance by Studying an Entire Class of Prediction Models
  Simultaneously.}
\newblock \emph{Journal of Machine Learning Research}, \textbf{20}(177), 1--81.

\bibitem[{Friedman(2001)}]{friedman2001greedy}
Friedman JH (2001).
\newblock \enquote{Greedy Function Approximation: A Gradient Boosting Machine.}
\newblock \emph{The Annals of Statistics}, \textbf{29}(5), 1189--1232.
\newblock \doi{10.1214/aos/1013203451}.

\bibitem[{Gr{\"o}mping(2006)}]{groemping2006relative}
Gr{\"o}mping U (2006).
\newblock \enquote{Relative Importance for Linear Regression in {R}: The
  Package {relaimpo}.}
\newblock \emph{Journal of Statistical Software}, \textbf{17}(1), 1--27.
\newblock \doi{10.18637/jss.v017.i01}.

\bibitem[{Harris \emph{et~al.}(2020)Harris, Millman, Van Der~Walt, Gommers,
  Virtanen, Cournapeau, Wieser, Taylor, Berg, Smith
  \emph{et~al.}}]{harris2020array}
Harris CR, Millman KJ, Van Der~Walt SJ, Gommers R, Virtanen P, Cournapeau D,
  Wieser E, Taylor J, Berg S, Smith NJ, \emph{et~al.} (2020).
\newblock \enquote{Array Programming with NumPy.}
\newblock \emph{Nature}, \textbf{585}(7825), 357--362.

\bibitem[{Hunter(2007)}]{Hunter:2007}
Hunter JD (2007).
\newblock \enquote{Matplotlib: A 2D Graphics Environment.}
\newblock \emph{Computing in Science \& Engineering}, \textbf{9}(3), 90--95.
\newblock \doi{10.1109/MCSE.2007.55}.

\bibitem[{Jiang \emph{et~al.}(2025)Jiang, Zhang, and
  Zhang}]{pmlr-v286-jiang25a}
Jiang Z, Zhang M, Zhang D (2025).
\newblock \enquote{Fast Calculation of Feature Contributions in Boosting
  Trees.}
\newblock In S~Chiappa, S~Magliacane (eds.), \emph{Proceedings of the
  Forty-first Conference on Uncertainty in Artificial Intelligence}, volume 286
  of \emph{Proceedings of Machine Learning Research}, pp. 1859--1875. PMLR.
\newblock \urlprefix\url{https://proceedings.mlr.press/v286/jiang25a.html}.

\bibitem[{Jullum \emph{et~al.}(2025)Jullum, Olsen, Lachmann, and
  Redelmeier}]{jullum2026shapr}
Jullum M, Olsen LHB, Lachmann J, Redelmeier A (2025).
\newblock \enquote{{shapr}: Explaining Machine Learning Models with Conditional
  Shapley Values in {R} and {Python}.}
\newblock \emph{arXiv preprint arXiv:2504.01842}.
\newblock \doi{10.48550/arXiv.2504.01842}.
\newblock \urlprefix\url{https://arxiv.org/abs/2504.01842}.

\bibitem[{Karczmarz \emph{et~al.}(2022)Karczmarz, Michalak, Mukherjee,
  Sankowski, and Wygocki}]{karczmarz2022improved}
Karczmarz A, Michalak T, Mukherjee A, Sankowski P, Wygocki P (2022).
\newblock \enquote{{Improved Feature Importance Computation for Tree Models
  based on the Banzhaf Value}.}
\newblock In \emph{Proceedings of the Thirty-Eight Conference on Uncertainty in
  Artificial Intelligence}, pp. 969--979.

\bibitem[{Ke \emph{et~al.}(2017)Ke, Meng, Finley, Wang, Chen, Ma, Ye, and
  Liu}]{ke2017lightgbm}
Ke G, Meng Q, Finley T, Wang T, Chen W, Ma W, Ye Q, Liu TY (2017).
\newblock \enquote{{LightGBM: A Highly Efficient Gradient Boosting Decision
  Tree}.}
\newblock \emph{Advances in Neural Information Processing Systems},
  \textbf{30}.

\bibitem[{Lipovetsky and Conklin(2001)}]{lipovetsky2001analysis}
Lipovetsky S, Conklin M (2001).
\newblock \enquote{Analysis of Regression in Game Theory Approach.}
\newblock \emph{Applied Stochastic Models in Business and Industry},
  \textbf{17}(4), 319--330.

\bibitem[{Lundberg \emph{et~al.}(2020)Lundberg, Erion, Chen, DeGrave, Prutkin,
  Nair, Katz, Himmelfarb, Bansal, and Lee}]{lundberg2020local}
Lundberg SM, Erion G, Chen H, DeGrave A, Prutkin JM, Nair B, Katz R, Himmelfarb
  J, Bansal N, Lee SI (2020).
\newblock \enquote{From Local Explanations to Global Understanding with
  Explainable AI for Trees.}
\newblock \emph{Nature Machine Intelligence}, \textbf{2}(1), 56--67.

\bibitem[{Lundberg and Lee(2017)}]{lundberg2017unified}
Lundberg SM, Lee SI (2017).
\newblock \enquote{{A Unified Approach to Interpreting Model Predictions}.}
\newblock \emph{Advances in Neural Information Processing Systems},
  \textbf{30}.

\bibitem[{Mayer(2024)}]{mayer2024shapviz}
Mayer M (2024).
\newblock \emph{{shapviz}: {SHAP} Visualizations}.
\newblock R package version 0.9.6,
  \urlprefix\url{https://CRAN.R-project.org/package=shapviz}.

\bibitem[{Mayer \emph{et~al.}(2026)Mayer, Komisarczyk, Kozminski, Maksymiuk,
  and Biecek}]{mayer2026treeshap}
Mayer M, Komisarczyk K, Kozminski P, Maksymiuk S, Biecek P (2026).
\newblock \emph{{treeshap}: Compute {SHAP} Values for Your Tree-Based Models
  Using the {TreeSHAP} Algorithm}.
\newblock \doi{10.32614/CRAN.package.treeshap}.
\newblock R package version 0.4.0,
  \urlprefix\url{https://CRAN.R-project.org/package=treeshap}.

\bibitem[{Mohammadi \emph{et~al.}(2026)Mohammadi, Reznikov, Sinitcyn, Muandet,
  and Chau}]{mohammadi2026quadrashap}
Mohammadi M, Reznikov G, Sinitcyn P, Muandet K, Chau SL (2026).
\newblock \enquote{QuadraSHAP: Stable and Scalable Shapley Values for Product
  Games via Gauss-Legendre Quadrature.}
\newblock \emph{arXiv preprint arXiv:2605.05870}.

\bibitem[{Owen and Prieur(2017)}]{owen2017shapley}
Owen AB, Prieur C (2017).
\newblock \enquote{On Shapley Value for Measuring Importance of Dependent
  Inputs.}
\newblock \emph{SIAM/ASA Journal on Uncertainty Quantification}, \textbf{5}(1),
  986--1002.
\newblock \doi{10.1137/16M1097717}.

\bibitem[{Pedregosa \emph{et~al.}(2011)Pedregosa, Varoquaux, Gramfort, Michel,
  Thirion, Grisel, Blondel, Prettenhofer, Weiss, Dubourg
  \emph{et~al.}}]{pedregosa2011scikit}
Pedregosa F, Varoquaux G, Gramfort A, Michel V, Thirion B, Grisel O, Blondel M,
  Prettenhofer P, Weiss R, Dubourg V, \emph{et~al.} (2011).
\newblock \enquote{Scikit-learn: Machine Learning in Python.}
\newblock \emph{Journal of Machine Learning Research}, \textbf{12}, 2825--2830.

\bibitem[{Prokhorenkova \emph{et~al.}(2018)Prokhorenkova, Gusev, Vorobev,
  Dorogush, and Gulin}]{prokhorenkova2018catboost}
Prokhorenkova L, Gusev G, Vorobev A, Dorogush AV, Gulin A (2018).
\newblock \enquote{CatBoost: Unbiased Boosting with Categorical Features.}
\newblock \emph{Advances in Neural Information Processing Systems},
  \textbf{31}.

\bibitem[{Shapley(1953)}]{shapley1953value}
Shapley LS (1953).
\newblock \enquote{{A Value for N-person Games}.}
\newblock \emph{Contributions to the Theory of Games}, \textbf{2(28)},
  307–317.

\bibitem[{Song \emph{et~al.}(2016)Song, Nelson, and Staum}]{song2016shapley}
Song E, Nelson BL, Staum J (2016).
\newblock \enquote{{Shapley Effects for Global Sensitivity Analysis: Theory and
  Computation}.}
\newblock \emph{SIAM/ASA Journal on Uncertainty Quantification}, \textbf{4}(1),
  1060--1083.

\bibitem[{Wettenstein \emph{et~al.}(2026)Wettenstein, Mitchell, and
  Yu}]{wettenstein2026quadrature}
Wettenstein R, Mitchell R, Yu P (2026).
\newblock \enquote{Quadrature-TreeSHAP: Depth-Independent TreeSHAP and Shapley
  Interactions.}
\newblock \emph{arXiv preprint arXiv:2605.04497}.

\bibitem[{Wickham(2016)}]{ggplot2}
Wickham H (2016).
\newblock \emph{ggplot2: Elegant Graphics for Data Analysis}.
\newblock Springer-Verlag New York.
\newblock ISBN 978-3-319-24277-4.
\newblock \urlprefix\url{https://ggplot2.tidyverse.org}.

\bibitem[{Williamson and Feng(2020)}]{williamson2020efficient}
Williamson B, Feng J (2020).
\newblock \enquote{{Efficient Nonparametric Statistical Inference on Population
  Feature Importance using Shapley Values}.}
\newblock In \emph{International Conference on Machine Learning}, pp.
  10282--10291. PMLR.

\bibitem[{Yang(2021)}]{yang2021fast}
Yang J (2021).
\newblock \enquote{{Fast Treeshap: Accelerating SHAP Value Computation for
  Trees}.}
\newblock \emph{Advances in Neural Information Processing Systems},
  \textbf{34}.

\end{thebibliography}


\newpage

\begin{appendix}

\end{appendix}


\end{document}